%% file: main.tex
\documentclass{article}

\usepackage{microtype}
\usepackage{graphicx}
\usepackage{subfigure}
\usepackage{booktabs} 
\usepackage{amsthm}

\usepackage{arydshln}
\usepackage{relsize}

\usepackage{hyperref}

\input{math_commands}

\input{misc_commands.tex}

\newcommand\bbE{\ensuremath{\mathbb{E}}} 

\usepackage[accepted]{icml2019}

\icmltitlerunning{Calibrated Model-Based Deep Reinforcement Learning}

\begin{document}

\twocolumn[
\icmltitle{Calibrated Model-Based Deep Reinforcement Learning}




\icmlsetsymbol{equal}{*}

\begin{icmlauthorlist}
\icmlauthor{Ali Malik}{equal,st}
\icmlauthor{Volodymyr Kuleshov}{equal,st,af}
\icmlauthor{Jiaming Song}{st}
\icmlauthor{Danny Nemer}{af}
\icmlauthor{Harlan Seymour}{af}
\icmlauthor{Stefano Ermon}{st}
\end{icmlauthorlist}

\icmlcorrespondingauthor{Ali Malik}{malikali@stanford.edu}
\icmlcorrespondingauthor{Volodymyr Kuleshov}{kuleshov@cs.stanford.edu}

\icmlaffiliation{st}{Department of Computer Science, Stanford University, USA}
\icmlaffiliation{af}{Afresh Technologies, San Francisco, USA}

\icmlkeywords{Machine Learning, ICML}

\vskip 0.3in
]






\printAffiliationsAndNotice{\icmlEqualContribution} 

\begin{abstract}

Estimates of predictive uncertainty are important for accurate model-based planning and reinforcement learning. However, predictive uncertainties
 --- especially ones derived from modern deep learning systems
 --- can be inaccurate and impose a bottleneck on performance. This paper explores which uncertainties are needed for model-based reinforcement learning and argues that good uncertainties must be calibrated, i.e. their probabilities should match empirical frequencies of predicted events. We describe a simple way to augment any model-based reinforcement learning agent with a calibrated model and show that doing so consistently improves planning, sample complexity, and exploration. On the \textsc{HalfCheetah} MuJoCo task, our system achieves state-of-the-art performance using 50\% fewer samples than the current leading approach.
Our findings suggest that calibration can improve the performance of model-based reinforcement learning with minimal computational and implementation overhead.

\end{abstract}

\input{1_intro.tex}

\input{2_background.tex}
\input{3_calibrated_mbrl.tex}

\input{4_applications.tex}

\input{5_experiments.tex}
\input{6_discussion.tex}
\input{7_related.tex}
\input{8_conclusion.tex}

\bibliography{main}
\bibliographystyle{icml2019}

\appendix
\newpage
\input{9_appendix.tex}

\end{document}

%% file: math_commands.tex
\usepackage{amsmath,amsfonts,bm,amssymb}

\newtheorem{theorem}{Theorem}

\newtheorem{lemma}{Lemma}

\newcommand{\TM}{\widehat{T}}
\newcommand{\TT}{T}

\def\eqref#1{equation~\ref{#1}}

\def\1{\bm{1}}

\DeclareMathAlphabet{\mathsfit}{\encodingdefault}{\sfdefault}{m}{sl}
\SetMathAlphabet{\mathsfit}{bold}{\encodingdefault}{\sfdefault}{bx}{n}

\def\gN{{\mathcal{N}}}

\newcommand{\E}{\mathbb{E}}

\newcommand{\Var}{\mathrm{Var}}

\DeclareMathOperator*{\argmax}{arg\,max}


%% file: misc_commands.tex
\usepackage[dvipsnames]{xcolor}

%% file: 1_intro.tex
\section{Introduction}

Methods for accurately assessing predictive uncertainty are important components of modern decision-making systems. Probabilistic methods have been used to improve the safety, interpretability, and performance of decision-making agents in various domains, including medicine~\citep{saria_sepsis_2018},
robotics~\citep{chua2018,buckman2018sample}, and operations research~\citep{van1997neuro}. 

In model-based reinforcement learning --- a setting in which an agent learns a model of the world from past experience and uses it to plan future decisions ---  capturing uncertainty in the agent's model is particularly important \cite{deisenroth2011pilco}.
Planning with a probabilistic model improves performance and sample complexity, especially when representing the model using a deep neural network~\cite{rajeswaran2016epopt,chua2018}.



\begin{figure}
    \centering
    \includegraphics[width=0.9\linewidth]{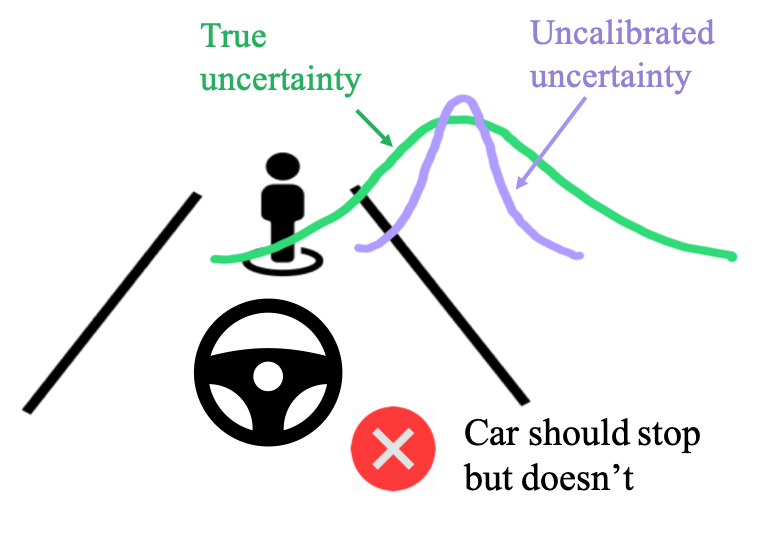}
    \caption{Modern model-based planning algorithms with probabilistic models can over-estimate their confidence (purple distribution), and overlook dangerous outcomes (e.g., a collision). We show how to endow agents with a calibrated world model that accurately captures true uncertainty (green distribution) and improves planning in high-stakes scenarios like autonomous driving or industrial optimisation.}
    \label{fig:calib_car_fig}
\end{figure}
 
Despite their importance in decision-making, predictive uncertainties can be unreliable, especially when derived from deep neural networks \cite{guo17}. 
Although several modern approaches such as deep ensembles~\citep{lakshminarayanan2017simple} and approximations of Bayesian inference~\citep{gal2016dropout,gal2016theoretically,gal2017concrete} provide uncertainties from deep neural networks, these methods suffer from shortcomings that reduce their effectiveness for planning \cite{kuleshov18}.


In this paper, we study which uncertainties are needed in model-based reinforcement learning and argue that good predictive uncertainties must be calibrated,
i.e. their probabilities should match empirical frequencies of predicted events. 
We propose a simple way to augment any model-based reinforcement learning algorithm with a calibrated model by adapting recent advances in uncertainty estimation for deep neural networks \cite{kuleshov18}. 
We complement our approach with diagnostic tools, best-practices, and intuition on how to apply calibration in reinforcement learning.

We validate our approach on benchmarks for contextual bandits and continuous control~\citep{Li2010,todorov2012mujoco}, as well as on a planning problem in inventory management~\citep{van1997neuro}. Our results show that calibration consistently improves the cumulative reward and the sample complexity of model-based agents, and also enhances their ability to balance exploration and exploitation in contextual bandit settings.
Most interestingly, on the {\sc HalfCheetah} task, our system achieves state-of-the-art performance, using 50\% 
fewer samples than the previous leading approach \citep{chua2018}. 
Our results suggest that calibrated uncertainties have the potential to improve model-based reinforcement learning algorithms with minimal computational and implementation overhead.

\paragraph{Contributions.}
In summary, this paper adapts recent advances in uncertainty estimation for deep neural networks to reinforcement learning and proposes a simple way to improve any model-based algorithm with calibrated uncertainties. We explain how this technique improves the accuracy of planning and the ability of agents to balance exploration and exploitation. 
Our method consistently improves performance on several reinforcement learning tasks, including contextual bandits, inventory management and  continuous control\footnote{Our code is available at \url{https://github.com/ermongroup/CalibratedModelBasedRL}}.

%% file: 2_background.tex
\section{Background}\label{sec:background}

\subsection{Model-Based Reinforcement Learning}


Let $\mathcal{S}$ and $\mathcal{A}$ denote (possibly continuous) state and action spaces in a Markov Decision Process $(\mathcal{S}, \mathcal{A}, T, r)$
and let $\Pi$ denote the set of all stationary stochastic policies $\pi:\mathcal{S} \rightarrow \mathcal{P}(\mathcal{A})$ that choose actions in $\mathcal{A}$ given states in $\mathcal{S}$. The successor state $s'$ for a given action $a$ from current state $s$ are drawn from the dynamics function $T(s'|s,a)$. We 
work in the $\gamma$-discounted infinite horizon setting 
and 
we will use an expectation with respect to a policy $\pi\in\Pi$ to denote an expectation with respect to the trajectory it generates: $\bbE_{\pi}[r(s,a)] \triangleq \bbE\left[\sum_{t=0}^\infty \gamma^t r(s_t,a_t)\right]$, where $s_0 \sim p_0$, $a_t \sim \pi(\cdot|s_t)$, and $s_{t+1}\sim \TT(\cdot|s_t,a_t)$ for $t\geq 0$. $p_0$ is the initial distribution over states and $r(s_t,a_t)$ is the reward at time $t$.

Typically, $\mathcal{S},\mathcal{A}, \gamma$ are known, while the dynamics model $T(s'|s,a)$ and the reward function $r(s,a)$ are not known explicitly. This work focuses on model-based reinforcement learning, in which the agent learns an approximate model $\hat T(s'|s,a)$ of the world from samples obtained by interacting with the environment and uses this model to plan its future decisions.



\paragraph{Probabilistic Models}

This paper focuses on probabilistic dynamics models $\TM(s'|s,a)$ that take a current state $s \in \mathcal{S}$ and action $a \in \mathcal{A}$, and output a probability distribution over future states $s'$. We represent the output distribution over the next states, $\TM(\cdot|s,a)$, as a cumulative distribution function $F_{s, a} : \mathcal{S} \to [0,1]$, which is defined for both discrete and continuous $\mathcal{S}$. 




\subsection{Calibration, Sharpness, and Proper Scoring Rules}


A key desirable property of probabilistic forecasts is calibration. Intuitively, a transition model $\TM(s'|s,a)$ 
is calibrated if whenever it assigns a probability of $0.8$ to an event --- such as a state transition $(s,a,s')$ --- that transition should occur about 80\% of the time. 



Formally, for a discrete state space $\mathcal{S}$ and when $s,a,s'$ are i.i.d.~realizations of random variables $S, A, S' \sim \mathbb{P}$, we say that a transition model $\TM$ is calibrated if
$$ P(S' = s' \mid \TM(S' = s'|S,A) = p) = p $$
for all $s' \in \mathcal{S}$ and $p \in [0,1]$.

When $\mathcal{S}$ is a continuous state space, calibration is defined using quantiles as
$
P(S' \leq F_{S,A}^{-1}(p)) = p \; \textrm{for all $p \in [0,1]$},
$
where $F_{s, a}^{-1}(p) = \inf\{y : p \leq F_{s, a}(y)\}$ is the quantile function associated with the CDF $F_{s, a}$ over future states $s'$ \cite{gneiting2007probabilistic}. A multivariate extension can be found in \citet{kuleshov18}.


Note that calibration alone is not enough for a model to be good. For example, assigning the same average probability to each transition may suffice as a calibrated model, but this model will not be useful. Good models need to also be sharp: intuitively, their probabilities should be maximally certain, i.e. close to 0 or 1. 

\paragraph{Proper Scoring Rules.}

In the statistics literature, probabilistic forecasts are typically assessed using proper scoring rules \cite{murphy1973vector,dawid1984prequential}. 
An example is the Brier score $L(p,q) = (p-q)^2$ defined over two Bernoulli distributions with natural parameters $p, q \in [0,1]$. 
Crucially, any proper scoring rule decomposes precisely into a calibration and a sharpness term \cite{murphy1973vector}:
$$
\text{Proper Scoring Rule} = \text{Calibration} + \text{Sharpness} + \text{const}.
$$
Most loss functions for probabilistic forecasts over both discrete and continuous variables are proper scoring rules \cite{gneiting2007strictly}. Hence, calibration and sharpness are precisely the two sufficient properties of a good forecast.




\subsection{Recalibration}

Most predictive models are not calibrated out-of-the-box \cite{niculescu2005predicting}.
However, given an arbitrary pre-trained forecaster $H: \mathcal{X} \to (\mathcal{Y} \to [0, 1])$ 
that outputs CDFs $F$, we may train an auxiliary model $R : [0,1] \to [0,1]$ such that the forecasts $R \circ F$ are calibrated in the limit of enough data. This recalibration procedure applies to any probabilistic regression model and does not worsen the original forecasts from $H$ when measured using a proper scoring rule~\citep{kuleshov2017estimating}. 

    
    

When $\mathcal S$ is discrete, a popular choice of $R$ is Platt scaling~\citep{platt1999probabilistic}; \citet{kuleshov18} extends Platt scaling to continuous variables. Either of these methods can be used within our framework. 

%% file: 3_calibrated_mbrl.tex
\section{What Uncertainties Do We Need In Model-Based Reinforcement Learning?}

In model-based reinforcement learning, probabilistic models improve the performance and sample complexity of planning algorithms~\cite{rajeswaran2016epopt,chua2018}; this naturally raises the question of what constitutes a good probabilistic model.

\subsection{Calibration vs. Sharpness Trade-Off}

A natural way of assessing the quality of a probabilistic model is via a proper scoring rule \cite{murphy1973vector,gneiting2007probabilistic}. As discussed in Section \ref{sec:background}, any proper scoring rule decomposes into a calibration and a sharpness term. Hence, these are precisely the two qualities we should seek.

Crucially, not all probabilistic predictions with the same proper score are equal: some are better calibrated, and others are sharper. There is a natural trade-off between these terms. 

In this paper, we argue that this trade-off plays an important role when specifying probabilistic models in reinforcement learning. Specifically, it is much better to be calibrated than sharp, and calibration significantly impacts the performance of model-based algorithms. Recalibration methods \cite{platt1999probabilistic,kuleshov18} allow us to ensure that a model is calibrated, and thus improve reinforcement learning agents.

\subsection{Importance of Calibration for Decision-Making}

In order to explain the importance of calibration, we provide some intuitive examples, and then prove a formal statement.

\paragraph{Intuition.}

Consider a simple MDP with two states $s_\text{good}$ and $s_\text{bad}$. The former has a high reward $r(s_\text{good}) = 1$ and the latter has a low reward $r(s_\text{bad}) = -1$. 

First, calibration helps us better estimate expected rewards. Consider the expected reward $\hat r$ from taking action $a$ in $s_\text{good}$ under the model. It is given by
$\hat r=  -1 \cdot \hat T(s_\text{bad}|s_\text{good}, a) + 1 \cdot \hat T(s_\text{good}|s_\text{good}, a). $
If the true transition probability is $T(s_\text{good}|s_\text{good}, a) = 80\%$, but our model $\hat T$ predicts 60\%, then in the long run the average reward from $a$ in $s_\text{good}$ will not equal to $\hat r$; incorrectly estimating the reward will in turn cause us to choose sub-optimal actions.
   
Similarly, suppose that the model is over-confident and $\hat T(s_\text{good}|s_\text{good}, a) = 0$; intuitively, we may decide that it is not useful to try $a$ in $s_\text{good}$, as it leads to $s_\text{bad}$ with 100\% probability. 
This is an instance of the classical exploration-exploitation problem; many approaches to this problem (such as the UCB family of algorithms) rely on accurate confidence bounds and are likely to benefit from calibrated uncertaintites that more accurately reflect the true probability of transitioning to a particular state.

\paragraph{Expectations Under Calibrated Models.}

More concretely, we can formalise our intuition about the accuracy of expectations via the following statement for discrete variables; see the Appendix for more details.

\begin{lemma}
Let $Q(Y | X)$ be a calibrated model over two discrete variables $X, Y \sim \mathbb{P}$ such that  $\mathbb{P}(Y = y \mid Q(Y = y \mid X) = p) = p$. Then any expectation of a function $G(Y)$ is the same under $P$ and $Q$:
\begin{equation}
    \mathlarger{\mathop{\mathbb{E}}}_{y \sim \mathbb{P}(Y)} \left[ G(y) \right] \quad = \mathlarger{\mathop{\mathbb{E}}}_{\substack{x \sim \mathbb{P}(X) \\ y \sim Q(Y | X = x)}} \left[G(y)\right].
\end{equation}
\end{lemma}

In model-based reinforcement learning, we take expectations in order to compute the expected reward of a sequence of decisions. A calibrated model will allow us to estimate these more accurately.

\section{Calibrated Model-Based Reinforcement Learning}

In Algorithm \ref{alg:recal}, we present a simple procedure that
augments a model-based reinforcement learning algorithm with an extra step that ensures the calibration of its transition model.
Algorithm \ref{alg:recal} effectively corresponds to standard model-based reinforcement learning with the addition of Step 4, in which we train a recalibrator $R$ such that $R \circ T$ is calibrated. The subroutine $\textsc{Calibrate}$ can be an instance of Platt scaling, for discrete $S$, or the method of \citet{kuleshov18}, when $S$ is continuous (see Algorithm \ref{alg:kul_cal} in the appendix).

\begin{algorithm}
  \caption{{\footnotesize Calibrated Model-Based Reinforcement Learning}}
  \label{alg:recal}
  \textbf{Input:}
    Initial transition model $\TM: \mathcal{S} \times \mathcal{A} \to \mathcal{P}(\mathcal{S})$ and 
    initial dataset of state transitions $\mathcal{D} = \{(s_t, a_t), s_{t+1}\}_{t=1}^{N}$\\
    Repeat until sufficient level of performance is reached:
   \begin{enumerate}
   \vspace{-2mm}
    \item Run the agent and collect a dataset of state transitions $\mathcal{D}_\text{new} \gets \textsc{ExecutePlanning}(\TM)$. Gather all experience data $\mathcal{D} \gets \mathcal{D} \cup \mathcal{D}_\text{new}$.
    \vspace{-2mm}
    \item Let $\mathcal{D}_\text{train}, \mathcal{D}_\text{cal} \gets \textsc{PartitionData}(\mathcal{D})$ be the training and calibration sets, respectively.
    \vspace{-2mm}
    \item Train a transition model $\TM \gets \textsc{TrainModel}(\mathcal{D}_\text{train})$.
    \vspace{-2mm}
    \item Train the recalibrator $R \gets \textsc{Calibrate}(\TM, \mathcal{D}_\text{cal})$.
    \vspace{-2mm}
    \item Let $\TM \gets R \circ \TM$ be the new, recalibrated transition model.
   \end{enumerate}
\end{algorithm}

In the rest of this section, we describe best practices for applying this method.


\paragraph{Diagnostic Tools.} 


An essential tool for visualising calibration of predicted CDFs $F_1, \ldots F_N$ is the reliability curve \cite{gneiting2007probabilistic}. This plot displays the empirical frequency of points in a given interval relative to the predicted fraction of points in that interval. Formally, we choose $m$ thresholds $0 \leq p_1 \leq \cdots \leq p_m \leq 1$ and, for each threshold $p_j$, compute the empirical frequency 
$
    \hat{p}_j = {|y_t : F_t(y_t) \leq p_j, t = 1, \ldots, N|}/{N}
$

Plotting $\{(p_j, \hat{p}_j)\}$ gives us a sense of the calibration of the model  (see Figure \ref{fig:linucb_avg_rewards}), with a straight line corresponding to perfect calibration. An equivalent, alternative visualisation is to plot a histogram of of the probability integral transform $\{F_t(y_t)\}_{t=1}^N$ and see if it looks like a uniform distribution \cite{gneiting2007probabilistic}.

These visualisations can be quantified by defining the calibration loss\footnote{This is the calibration term in the two-component decomposition of the Brier score.} of a model:
\begin{align}
    \text{cal}(F_1, y_1, \ldots F_t,  y_Y) = \sum_{j = 1}^m (\hat{p}_j - p_j)^2, \label{eq:cal_loss}
\end{align}
as the sum of the squares of the residuals $(\hat{p}_j - p_j)^2$. These diagnostic tools should be evaluated on unseen data distinct from the training and calibration sets as it may reveal signs of overfitting.




\subsection{Applications to Deep Reinforcement Learning}

Although deep neural networks can significantly improve model-based planning algorithms~\citep{higuera2018synthesizing,chua2018}, their estimates of predictive uncertainty are often inaccurate \cite{guo17, kuleshov18}. 

\paragraph{Variational Dropout.}
One popular approach to deriving uncertainty estimates from deep neural networks involves using dropout.
Taking the mean and the variance of dropout samples leads to a principled Gaussian approximation of the posterior predictive distribution from a Bayesian neural network (in regression)~\citep{gal2016dropout}. To use Algorithm \ref{alg:recal} we may instantiate {\sc Calibrate} with the method of \citet{kuleshov18} and pass  it the predictive Gaussian derived from the dropout samples.

More generally, our method can be naturally applied on top of any probabilistic model without any need to modify or retrain this model. 


%% file: 4_applications.tex
\section{The Benefits of Calibration in Model-Based Reinforcement Learning}

Next, we examine specific ways in which Algorithm \ref{alg:recal} can improve model-based reinforcement learning agents.

\subsection{Model-Based Planning}

The first benefit of a calibrated model is enabling more accurate planning using standard algorithms such as value iteration or model predictive control~\citep{sutton2018reinforcement}.
Each of these methods involves estimates of future reward. For example, value iteration performs the update
$$ V'(s) \gets \mathbb{E}_{a \sim \pi(\cdot|s)}\left[ \sum_{s' \in S} \hat T(s'|s,a)(r(s') + V(s'))  \right].$$
Crucially, this algorithm requires accurate estimates of the expected reward $\sum_{s' \in S} \hat T(s'|s,a)r(s')$. Similarly, online planning algorithms involve computing the expected reward of a finite sequence of actions, which has a similar form.
If the model is miscalibrated, then the predicted distribution $\hat T(s'|s,a)$ will not accurately reflect the true distribution of states that the agent will encounter in the real world. As a result, planning performed in the model will be inaccurate. 

More formally, let us define the value of a policy $\pi$ as $V(\pi) = \mathbb{E}_{s \sim \sigma_\pi} [V(s)]$, where $\sigma_\pi$ is the stationary distribution of the Markov chain induced by $\pi$. Let $V'(\pi)$ be an estimate of the value of $\pi$ in a second MDP in which we replaced the transition dynamics by a calibrated model $\TM$ learned from data. Then, the following holds.

\begin{theorem}
Let  $(S, A, T, r)$ be a discrete MDP and let $\pi$ be a stochastic policy over this MDP. The value $V(\pi)$ of policy $\pi$ under the true dynamics $T$ is equal to the value $V'(\pi)$ of the policy under any set of calibrated dynamics $\TM$.
\end{theorem}

Effectively, having a calibrated model makes it possible to compute accurate expectations of rewards, which in turn provides accurate estimates of the values of states and policies. Accurately estimating the value of a policy makes it easier to choose the best one by planning.

\begin{table}[]
\small
\begin{tabular}{l|l|l|l}
\textit{}  & \textit{LinUCB}           & \textit{CalLinUCB} & \textit{Optimal} \\
\hline
Linear    & $1209.8 \pm 12.1$ &   $1210.3 \pm 12.1$    &    $1231.8$     \\
Beta      & $1176.3 \pm 11.9$ & $1174.6 \pm 12.0$           &  $1202.3$  \\\midrule
Mushroom  &     $1429.4  \pm 154.0$     &   $\mathbf{1676.1 \pm 164.1}$          &    $3122.0$     \\
Covertype &      $558.14 \pm 3.5$      &     $\mathbf{677.8 \pm 5.0}$      &    $1200.0$     \\
Adult     &  $131.3 \pm 1.2$          &   $\mathbf{198.9  \pm 4.7}$ &     $1200.0$    \\
Census    &       $207.6 \pm 1.7$           &         $\mathbf{603.7 \pm 3.8}$  &  $1200.0$
\end{tabular}
\caption{Performance of calibrated/uncalibrated LinUCB on a variety of datasets, averaged over $10$ trials. The calibrated algorithm (CalLinUCB) does better on all non-synthetic datasets (bottom four rows) and has similar performance on the synthetic datasets (top two rows).}
\label{table:linucb_perf}
\end{table}

\subsection{Balancing Exploration and Exploitation}

Balancing exploration and exploitation successfully is a fundamental challenge for many reinforcement learning (RL) algorithms. A large family of algorithms tackle this problem using notions of uncertainty/confidence to guide their exploration process. For example, upper confidence bound (UCB, \citet{auer2002finite}) algorithms pick the action which has the highest upper bound on its reward confidence interval.
    
In situations where the outputs of the algorithms are uncalibrated, the confidence intervals might provide unreliable upper confidence bounds, resulting in suboptimal performance. For example, in a two-arm bandit problem, if a model is under-estimating the reward of the best arm and has high confidence, it's upper confidence bound will be low, and it will not be selected. More generally, UCB-style methods need uncertainty estimates to be on the same ``order of magnitude" so that arms can be compared against each other; calibration helps ensure that.

%% file: 5_experiments.tex
\section{Experiments}

We evaluate our calibrated model-based reinforcement learning method on several different environments and algorithms, 
including contextual bandits, inventory management, and continuous control for robotics.

\subsection{Balancing Exploration and Exploitation}

        To test the effect of calibration on exploration/exploitation, we look at the contextual multi-armed bandit problem~\cite{Li2010}. At each timestep, an agent is shown a context vector $\mathbf{x}$ and must pick an arm $a \in \mathcal{A}$ from a finite set $\mathcal{A}$. After picking an arm, the agent receives a reward $r_{a, \mathbf{x}}$ which depends both on the arm picked and also on the context vector shown to the agent.  The agent's goal over time is to learn the relationship between the context vector and reward gained from each arm so that it can pick the arm with the highest expected reward at each timestep.

    \paragraph{Setup.} 
    For our experiments, we focus on the LinUCB algorithm \cite{Li2010} --- a well-known instantiation of the UCB approach to contextual bandits. LinUCB assumes a linear relationship between the context vector and the expected reward of an arm: for each arm $a \in \mathcal{A}$, there is an unknown coefficient vector $\theta_a^*$ such that $\E[r_{a, \mathbf{x}}] = \mathbf{x}^\top \theta_a^*$.  
    
    LinUCB learns a predictive distribution over this reward using Bayesian ridge regression, in which $\theta_a^*$ has a Gaussian posterior  $\mathcal{N}(\hat{\theta}_a, \hat{\Sigma}_a)$. The posterior predictive distribution is also Gaussian, with mean $\mathbf{x}^\top \hat{\theta}_a$ and with standard deviation $\sqrt{\mathbf{x}^\top \hat{\Sigma}_a^{-1} \mathbf{x}}$. Thus, the algorithm picks the arm with the highest $\alpha$-quantile, given by 
    
    \begin{equation}
     \argmax_{a \in \mathcal{A}}\left(\mathbf{x}^\top \hat{\theta}_a + \alpha \cdot \sqrt{\mathbf{x}^\top \hat{\Sigma}_a^{-1} \mathbf{x}}\right).
    \end{equation}
    
    We apply the recalibration scheme in Algorithm \ref{alg:recal} of \citet{kuleshov18} to these predicted Gaussian distributions.

    \paragraph{Data.} 
    We evaluate the calibrated version (CalLinUCB) and uncalibrated version (LinUCB) of the LinUCB algorithm on both synthetic data that satisfies the linearity assumption of the algorithm, as well as on real UCI datasets from \citet{Li2010}. We run the tests on $2000$ examples over $10$ trials and compute the average cumulative reward. 
    
    \paragraph{Results.}
    
    \begin{figure}[h]
        \centering
        \includegraphics[height=120pt,width=0.85\linewidth]{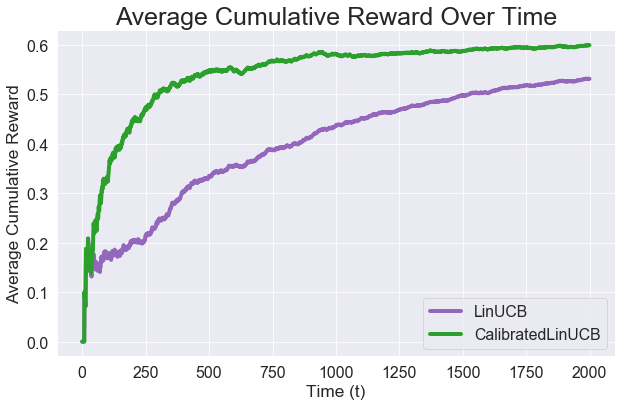}
        \includegraphics[height=120pt, width=0.85\linewidth]{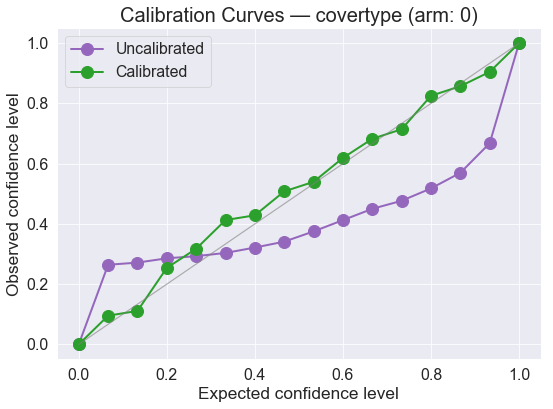}
        \caption{Top: Performnce of CalibLinUCB and LinUCB on the UCI covertype dataset. Bottom: Calibration curves of the LinUCB algorithms on the covertype dataset} 
        \label{fig:linucb_avg_rewards}
    \end{figure}

    We expect the LinUCB algorithm to already be calibrated on the synthetic linear data since the model is well-specified, implying no difference in performance between CalLinUCB and LinUCB. On the real UCI datasets however, the linear assumption might not hold, resulting in miscalibrated estimates of the expected reward. 

    In Table \ref{table:linucb_perf}, we can see that indeed there is no significant difference in performance between the CalLinUCB and LinUCB algorithms on the synthetic linear dataset---they both preform optimally. On the UCI datasets however, we see a noticeable improvement with CalLinUCB on almost all tasks, suggesting that recalibration aids exploration/exploitation in a setting where the model is misspecified. Note that both CalLinUCB and LinUCB perform below the optimum on these datasets, implying linear models are not expressive enough in general for these tasks.

    \paragraph{Analysis.} 
    To get a sense of the effect of calibration on the model's confidence estimates, we can plot the predicted reward with 90\% confidence intervals that the algorithm expected for a chosen arm $a$ at timestep $t$. We can then compare how good this prediction was with respect to the true observed reward. Specifically, we can look at the timesteps where the CalLinUCB algorithm picked the optimal action but the LinUCB algorithm did not, and look at both the algorithms' belief about the predicted reward from both of these actions. An example of this plot can be seen in the appendix on Figure \ref{fig:covertype_predictions}.
    
    A key takeaway from this plot is that the uncalibrated algorithm systematically underestimates the expected reward from the optimal action and overestimates the expected reward of the action it chose instead, resulting in suboptimal actions. The calibrated model does not suffer from this defect, and thus performs better on the task.

\subsection{Model-Based Planning}


    \subsubsection{Inventory Management}
    
    Our first model-based planning task is inventory management~\citep{van1997neuro}. 
    A decision-making agent controls the inventory of a perishable good in a store. Each day, the agent orders items into the store; if the agent under-orders, the store runs out of stock; if the agent over-orders, perishable items are lost due to spoilage.
    Perishable inventory management systems have the potential to positively impact the environment by minimizing food waste and enabling a more effective use of resources \cite{vermeulen2012climate}.
    
    \paragraph{Model.}
    We formalize perishable inventory management for one item using a Markov decision process $(\mathcal{S}, \mathcal{A}, P, r)$. States $s \in S$ are tuples $(d_s, q_s)$ consisting of a calendar day $d_s$ and an inventory state $q_s \in\mathbb{Z}^L$, where $L \geq 1$ is the item shelf-life. Each component $(q_s)_l$ indicates the number of units in the store that expire in $l$ days; the total inventory level is $t_s = \sum_{l=1}^L (q_s)_l$. Transition probabilities $P$ are defined as follows: each day sees a random demand of $D{(s)} \in \mathbb{Z}$ units and sales of $\max(t^{(s)}-D{(s)}, 0)$ units, sampled at random from all the units in the inventory; at the end of the state transition, the shelf-life of the remaining items is decreased by one (spoiled items are recorded and thrown away). Actions $a \in \mathcal{A} \subseteq \mathbb{Z}$ correspond to orders: the store receives items with a shelf life of $L$ before entering the next state $s'$.
    In our experiments we choose the reward $r$ to be the sum of waste and unmet demand due to stock-outs.
    
    \paragraph{Data.}
    
    We use the Corporacion Favorita Kaggle dataset, which consists of historical sales from a supermarket chain in Ecuador. We experiment on the 100 highest-selling items and use data from 2014-01-01 to 2016-05-31 for training and data from 2016-06-01 to 2016-08-31 for testing.
    
    \paragraph{Algorithms.}
    
    We learn a probabilistic model $\hat M : \mathcal{S} \to (\mathbb{R} \to [0,1])$ of the demand $D(s')$ in a future state $s'$ based on information available in the present state $s$. Specifically, we train a Bayesian
    DenseNet~\citep{huang2017densely} to predict sales on each of the next five days based on features from the current day (sales serve as a proxy for demand). We use autoregressive features from the past four days, 7-, 14-, and 28-day rolling means of historical sales, binary indicators for the day of the week and the week of the year, and sine and cosine features over the number of days elapsed in the year. The Bayesian DenseNet has five layers of 128 hidden units with a dropout rate of 0.5 and parametric ReLU non-linearities. We use variational dropout~\citep{gal2016theoretically} to compute probabilistic forecasts from the model.
    
    We use our learned distribution over $D(s')$ to perform online planning on the test set using model predictive control (MPC) learned on the training set. Specifically, we sample 5,000 random trajectories over a 5-step horizon, and choose the first action of the trajectory with the highest expected reward under the model. We estimate the expected reward of each trajectory using 300 Monte Carlo samples from the model.
    
    We also compare the planning approach to a simple heuristic rule that always sets the inventory to $1.5 \cdot \mathbb{E}[D(s')]$, which is the expected demand multiplied by a small safety factor.
    
    \paragraph{Results.}
    
        \begin{table}[]
        \small
        \begin{tabular}{l|l|l|l}
        \textit{}  & \textit{Calibrated}           & \textit{Uncalibrated} & \textit{Heuristic} \\
        \hline
        Shipped    & 332,150 &   319,692    &    338,011     \\
        Wasted      & 7,466 & 3,148           &  13,699  \\
        Stockouts  &     9,327     &   17,358          &    11,817     \\\midrule
        \% Waste &      2.2\%      &     1.0\%      &    4.1\%     \\
        \% Stockouts     &  2.8\%          &   5.4\% &     3.5\%    \\\midrule
        {\em Reward}  &     {\bf -16,793}     &   {-20,506}          &    {-25,516}     \\
        \end{tabular}
        \caption{\label{tab:inventory}Performance of calibrated model planning on an inventory management task. Calibration significantly improves cumulative reward. Numbers are in units, averaged over ten trials.}
        \end{table}      
    
    We evaluate the agent within the inventory management MDP; the demand $D(s)$ is instantiated with the historical sales on test day $d(s)$ (which the agent did not observe). We measure total cumulative waste and stockouts over the 100 items in the dataset, and we report them as a fraction of the total number of units shipped to the store.
    
    Table~\ref{tab:inventory} shows that calibration improves the total cumulative reward by 14\%. The calibrated model incurs waste and out-of-stocks ratios of 2.2\% and 2.8\%, respectively, compared to 1.0\% and 5.4\% for the uncalibrated one. 
    These values are skewed towards a smaller waste, while the objective function penalizes both equally.
    The heuristic has ratios of 4.1\% and 3.5\%.
        
    \subsubsection{Mujoco Environments}
        \begin{figure*}[h]
            \centering
            \includegraphics[width=0.46\linewidth]{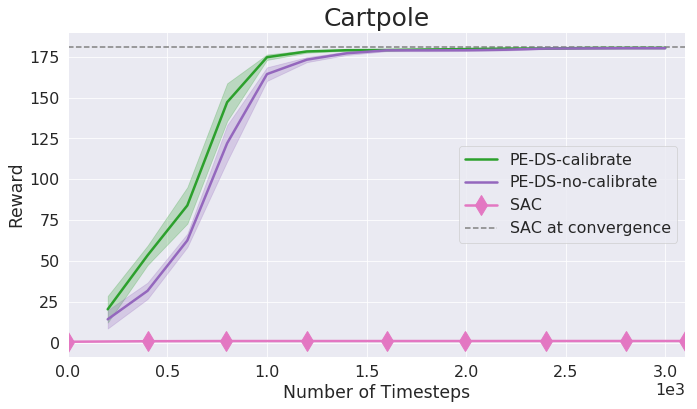} 
            \includegraphics[width=0.46\linewidth]{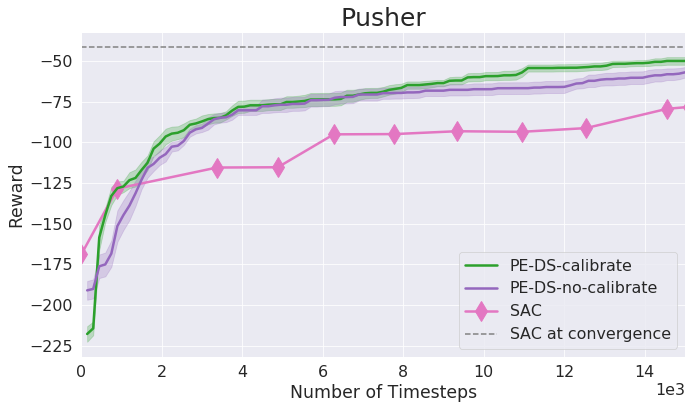}
            \includegraphics[width=0.48\linewidth]{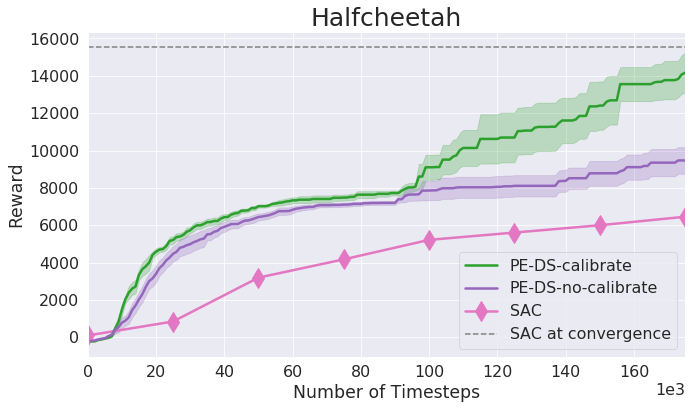}
            \caption{Performance on different control tasks. The calibrated algorithm does at least as good, and often much better than the uncalibrated models. Plots show maximum reward obtained so far, averaged over $10$ trials. Standard error is displayed as the shaded areas.}
            \label{fig:mbrl_perf}
        \end{figure*}    

Our second model-based planning task is continuous control from OpenAI Gym~\citep{brockman2016openai} and the Mujoco robotics simulation environment~\citep{todorov2012mujoco}. Here the agent makes decisions about its torque controls given observation states (e.g. location / velocity of joints) that maximizes the expected return reward. These environments are standard benchmark tasks for deep reinforcement learning.

\paragraph{Setup.} We consider calibrating the probablistic ensemble dynamics model proposed in \cite{chua2018}. In this approach, the agent learns an ensemble of probabilistic neural networks (PE) that captures the environment dynamics $s_{t+1} \sim f_\theta(s_t, a_t)$, which is used for model-based planning with model predictive control. The policy and ensemble model are then updated in an iterative fashion. \citet{chua2018} introduce several strategies for particle-based state propagation, including trajectory sampling with bootstraped models (PE-TS); and distribution sampling (PE-DS), which samples from a multimodal distribution as follows:
\begin{gather}
    s_{t+1} \sim \gN(\E[s^{p}_{t+1}], \Var[s^{p}_{t+1}]), \quad s^{p}_{t+1} \sim f_\theta(s_t, a_t)
\end{gather}
PE-TS and PE-DS achieve the highest sample efficiency among the methods proposed in~\citep{chua2018}.

To calibrate the model, we add a final sigmoid recalibration layer to the sampling procedure in PE-DS at each step. This logistic layer is applied separately per output state dimension and serves as the recalibrator $R$. It is trained on the procedure described in Algorithm \ref{alg:kul_cal}, after every trial, on a separate calibration set using cross entropy loss.

We consider three continuous control environments from \citet{chua2018} \footnote{We omitted the reacher environment because the reference papers did not have SAC results for it.}. For model learning and model-based planning, we follow the training procedure and hyperparameters in \citet{chua2018}, as described in \href{https://github.com/kchua/handful-of-trials}{https://github.com/kchua/handful-of-trials}. We also compare our method against Soft Actor-Critic~\citep{haarnoja2018soft} which is one of the state-of-the-art model-free reinforcement learning algorithms. We use the final convergence reward of SAC as a criterion for the highest possible reward achieved in the task (although it may require orders of magnitude more samples from the environment).

\paragraph{Results.} One of the most important criteria for evaluating reinforcement learning algorithms is sample complexity, i.e, the amount of interactions with the environment in order to reach a certain high expected return. We compare the sample complexities of SAC, PE-DS and calibrated PE-DS in Figure~\ref{fig:mbrl_perf}. Compared to the model-free SAC method, both the model-based methods use much fewer samples from the environment to reach the convergence performance of SAC. However, our recalibrated PE-DS method compares favorably to PE-DS on all three environments.

Notably, the calibrated PE-DS method outperforms both PE-DS by a significant margin on the HalfCheetah environment, reaching near optimal performance at only around 180k timesteps. To our knowledge, the calibrated PE-DS is the most efficient method on these environments in terms of sample complexity. 

\paragraph{Analysis.} In Figure \ref{fig:cartpole_predictions} in the appendix, we visualise the 1-step prediction accuracy for action dimension zero in the Cartpole environment for both PE-DS and calibrated PE-DS. This figure shows that the calibrated PE-DS model is more accurate, has tighter uncertainty bounds, and is better calibrated, especially in earlier trials. Interestingly, we also observe a superior expected return for calibrated PE-DS for earlier trials in Figure~\ref{fig:mbrl_perf}, suggesting that being calibrated is correlated with improvements in model-based prediction and planning.

%% file: 6_discussion.tex
\section{Discussion}


\paragraph{Limitations.}



A potential failure mode for our method arises when not all forecasts are from the same family of distributions. This can lead to calibrated, but diffuse confidence intervals. 
%
%
Another limitation of the method is its scalability to high-dimensional spaces. In our work, the uncalibrated forecasts were fully factored, and could be recalibrated component-wise. For non-factored distributions, recalibration is computationally intractable and requires approximations such as ones developed for multi-class classification \cite{zadrozny2002transforming}.

Finally, it is possible that uncalibrated forecasts are still effective if they induce a model that correctly ranks the agent's actions in terms of their expected reward (even when the estimates of the reward themselves are incorrect). 

\paragraph{Extensions to Safety.}
Calibration also plays an important role in the domain of RL safety \cite{Berkenkamp2017}. 
In situations where the agent is planning its next action, if it determines the 90\% confidence interval of the predicted next state to be in a safe area but this confidence is miscalibrated, then the agent has a higher chance of entering a failure state.


%% file: 7_related.tex
\section{Related Work}

\paragraph{Model-based Reinforcement Learning.} 

Model-based RL is effective in low-data and/or high-stakes regimes such as robotics~\citep{chua2018}, dialogue systems~\citep{singh2000reinforcement}, education~\citep{rollinson2015predictive}, scientific discovery~\citep{mcintire2016sparse}, or conservation planning~\citep{ermon2012playing}. A big challenge of model-based RL is the model bias, which is being addressed by solutions such as model ensembles~\citep{clavera2018model,kurutach2018model,depeweg2016learning,chua2018} or combining with model-free approaches~\citep{buckman2018sample}. 


\paragraph{Calibration.}
Two of the most widely used calibration procedures are Platt scaling \cite{platt1999probabilistic} and isotonic regression \cite{niculescu2005predicting}. They can be extended from binary to multi-class classification \cite{zadrozny2002transforming}, to structured prediction \cite{kuleshov2015calibrated}, and to regression \cite{kuleshov18}.
Calibration has recently been studied in the context of deep neural networks \cite{guo2017calibration,gal2017concrete,lakshminarayanan2016simple}, identifying important shortcomings in their uncertainties.

\paragraph{Probabilistic forecasting.}
Calibration has been studied extensively in statistics \cite{murphy1973vector,dawid1984prequential} as a criterion for evaluating forecasts \cite{gneiting2007strictly}, including from a Bayesian perspective
\citet{dawid1984prequential}.
Recent studies on calibration have focused on applications in weather forecasting \cite{gneiting2005weather}, and have led to implementations in forecasting systems \cite{raftery2005using}.  \citet{gneiting2007probabilistic} introduced a number of definitions of calibration for continuous variables, complementing early work on classification \cite{murphy1973vector}.

%% file: 8_conclusion.tex
\section{Conclusion}
Probabilistic models of the environment can significantly improve the performance of reinforcement learning agents. However, proper uncertainty quantification is crucial for planning and managing exploration/exploitation tradeoffs. We demonstrated a general recalibration technique that can be combined with most model-based reinforcement learning algorithms to improve performance. Our approach leads to minimal computational overhead, and empirically improves performance across a range of tasks.

\section*{Acknowledgments}
This research was supported by NSF (\#1651565, \#1522054, \#1733686), ONR (N00014-19-1-2145), AFOSR (FA9550-19-1-0024), Amazon AWS, and Lam Research.

%% file: 9_appendix.tex
\newpage
\onecolumn
\appendix
\section*{Appendices}

\section{Recalibration}

Most predictive models are not calibrated out-of-the-box due to modeling bias and computational approximations.
However, given an arbitrary pre-trained forecaster $H: \mathcal{X} \to (\mathcal{Y} \to [0, 1])$ 
that outputs CDFs $F$, we may train an auxiliary model $R : [0,1] \to [0,1]$ such that the forecasts $R \circ F$ are calibrated in the limit of enough data.  This procedure, called recalibration, is simple to implement, computationally inexpensive, and can be applied to any probabilistic regression model in a black-box manner. Furthermore, it does not increase the loss function of the original model if it belongs to a large family of objectives called proper losses~\citep{kull2015novel,kuleshov2017estimating}. 

\begin{algorithm}
  \caption{{\footnotesize \textsc{Calibrate}: Recalibration of Transition Dynamics}}
  \label{alg:kul_cal}
  \textbf{Input:}
    Uncalibrated transition model $\TM: \mathcal{S} \times \mathcal{A} \to \mathcal{P}(\mathcal{S})$ that outputs CDFs $F_{s,a} : \mathcal{S} \to [0,1]$, and calibration set $\mathcal{D}_\text{cal} = \{(s_t, a_t), s_{t+1}\}\}_{t=1}^{N}$ \\
    \textbf{Output:}
    Auxiliary recalibration model $R : [0,1] \to [0,1]$. \\
   \begin{enumerate}
   \vspace{-5mm}
    \item Construct a recalibration dataset
    \vspace{-2mm}
    \[
        \mathcal{D} = \left\{\left(F_{s_t, a_t}(s_{t+1}) , \hat{P}(F_{s_t, a_t}(s_{t+1}))\right) \right\}_{t=1}^N
    \]
    \vspace{-4mm}
    where
    \vspace{-1mm}
    \[
        \hat{P}(p) = \frac{1}{N} \sum_{t = 1}^N \mathbb{I}[F_{s_t, a_t}(s_{t+1}) \leq p].
    \]
    
    \vspace{-4mm}
    \item Train a model $R: [0,1] \to [0,1]$ (e.g. sigmoid or isotonic regression) on $\mathcal{D}$.
    
   \end{enumerate}
\end{algorithm}

When $\mathcal S$ is discrete, a popular choice of $R$ is Platt scaling~\citep{platt1999probabilistic}; \citet{kuleshov18} extends Platt scaling to continuous variables. Either of these methods can be used within our framework. 
Since this paper focuses on continuous state spaces, we use the method of \citet{kuleshov18} described in Algorithm \ref{alg:kul_cal}, unless otherwise indicated.

\section{Calibrated Discrete MDP}
We provide a proof in the discrete case that calibrated uncertainties result in correct expectations with respect to the true probability distribution, and thus using calibrated dynamics allow accurate evaluation of policies.

Consider an infinite-horizon discrete state MDP $(S, A, T, R)$ and a policy $\pi$ over this MDP. We are interested in evaluating the goodness of this policy at any state $s$ using the usual value iteration:

\begin{equation}
    V_{\pi}(s) = R(s) + \gamma \mathlarger{\mathop{\mathbb{E}}}_{a \sim \pi(\cdot | s)} \ \mathlarger{\mathop{\mathbb{E}}}_{s' \sim T(\cdot | s, a)} [V(s')].
\end{equation}

For the given policy $\pi$, there exists a stationary distribution $\sigma_\pi$ that would be obtained from running this policy for a long time. We define the value of the entire policy $V(\pi)$ as an expectation with respect to this stationary distribution i.e. 

\begin{equation}
    V(\pi) = \mathlarger{\mathop{\mathbb{E}}}_{s \sim \sigma_\pi} [V(s)].
\end{equation}

We want to show that replacing the true dynamics $T$ with calibrated dynamics $\TM$ does not affect our evaluation of the policy $\pi$. To have a well-defined notion of calibration, we need to define a joint distribution over the inputs and outputs of a predictive model. The inputs are the current state-action pair $(s, a)$ the outputs are distributions over the next state $s'$. To define a joint distribution $\mathbb{P}$ over  $(S, A)$ and $S'$, we use the stationary distribution $\sigma_\pi$, the policy $\pi$, and the transition dynamics $T$ to define the sub-components using the chain rule:

\begin{align}
    \mathbb{P}((s, a), s') 
        &= \mathbb{P}(s' | s, a) \ \mathbb{P}(a | s) \ \mathbb{P}(s) \\
        \label{eqn:joint_sas}
        &= T(s' | s, a) \ \pi(a | s) \ \sigma_\pi(s).
\end{align}

Note that defining the joint distribution $\mathbb{P}$ this way lets us rewrite $V(\pi)$ more simply as

\begin{align}
    V(\pi) 
        &= \mathlarger{\mathop{\mathbb{E}}}_{s \sim \sigma_\pi} [V(s)] \\
        &= \mathlarger{\mathop{\mathbb{E}}}_{s \sim \sigma_\pi}[R(s)] + \gamma \mathlarger{\mathop{\mathbb{E}}}_{s \sim \sigma_\pi} \mathlarger{\mathop{\mathbb{E}}}_{a \sim \pi(\cdot | s)} \ \mathlarger{\mathop{\mathbb{E}}}_{s' \sim T(\cdot | s, a)} [V(s')] \\
        &= \mathlarger{\mathop{\mathbb{E}}}_{s \sim \sigma_\pi}[R(s)] + \gamma \mathlarger{\mathop{\mathbb{E}}}_{((s, a), s') \sim \mathbb{P}} [V(s')] \\
        &= \mathlarger{\mathop{\mathbb{E}}}_{s \sim \sigma_\pi}[R(s)] + \gamma \mathlarger{\mathop{\mathbb{E}}}_{s' \sim \mathbb{P}(S')} [V(s')] 
\end{align}

These definitions allow us to state our theorem.

\begin{theorem}
Let  $(S, A, T, R)$ be a discrete MDP and let $\pi$ be a stochastic policy over this MDP. Define $\mathbb{P}$ to be a joint distribution over state-action and future state pairs $((s,a), s')$ as outlined in Equation \ref{eqn:joint_sas}. Then the value of policy $\pi$ under the true dynamics $T$ is equal to the value of the policy under some other dynamics $\TM$ that are calibrated with respect to $\mathbb{P}$.
\end{theorem}

\begin{proof}
    Since $\TM$ is calibrated with respect to $\mathbb{P}$, we have $\mathbb{P}(s' = j \mid \TM(s' = j | s, a) = p)) = p$. Let $\widehat{V}(\pi)$ be the value of policy $\pi$ under $\TM$. Then we have
    
    \begin{align}
        \widehat{V}(\pi) 
            &= \mathlarger{\mathop{\mathbb{E}}}_{s \sim \sigma_\pi}[R(s)] + \gamma \mathlarger{\mathop{\mathbb{E}}}_{(s, a) \sim \mathbb{P}((s, a))} \ \mathlarger{\mathop{\mathbb{E}}}_{s' \sim \TM(s' | s, a)} [V(y)] \\
            &= \mathlarger{\mathop{\mathbb{E}}}_{s \sim \sigma_\pi}[R(s)] + \gamma \mathlarger{\mathop{\mathbb{E}}}_{s' \sim \mathbb{P}(S')} [V(y)] \\
            &= V(\pi),
    \end{align}
    
    where the second line follows immediately from Lemma \ref{lem:calib_expectation_equal} when we take $x = (s, a)$ and $y = s'$.

\end{proof}

\begin{lemma}
\label{lem:calib_expectation_equal}
Consider a pair of jointly distributed variables $(X, Y) \sim \mathbb{P}$ over $\mathbb{X}$ and $\mathbb{Y}$ where $\mathbb{X} = \{x_1, \ldots, x_k\}$ and $\mathbb{Y} = \{y_1, \ldots, y_m\}$ are discrete spaces, and let $Q(Y | X)$ be a distribution that is calibrated with respect to $\mathbb{P}$. In other words,  $\mathbb{P}(Y = y \mid Q(Y = y \mid X) = p) = p$. Then, for any arbitrary function $g : \mathbb{Y} \to \mathcal{S}$ with which we want to take an expectation, the following equality holds:

\begin{equation}
    \label{eqn:calib_expectation}
    \mathlarger{\mathop{\mathbb{E}}}_{y \sim \mathbb{P}(Y)} \left[ g(y) \right] \quad = \mathlarger{\mathop{\mathbb{E}}}_{\substack{x \sim \mathbb{P}(X) \\ y \sim Q(Y | X = x)}} \left[g(y)\right].
\end{equation}
\end{lemma}

\begin{proof}
We can rewrite the expectation on the LHS of Equation \ref{eqn:calib_expectation} using the law of total probability and the chain rule to get:

\begin{align*}
    \mathlarger{\mathop{\mathbb{E}}}_{y \sim \mathbb{P}(Y)} \left[ g(y) \right] 
        &= \sum_{y \in \mathbb{Y}} g(y) \mathbb{P}(Y = y) \\
        &= \sum_{y \in \mathbb{Y}} g(y) \int_{0}^1 \mathbb{P}(Y = y, Q(Y = y \mid X) = p) \ dp \\
        &= \sum_{y \in \mathbb{Y}} g(y) \int_{0}^1 \mathbb{P}(Y = y \mid Q(Y = y \mid X) = p) \ \mathbb{P}(Q(Y = y \mid X) = p) \ dp.
\end{align*}

Note that in the above derivation we perform the following slight abuse of notation:
$$
\{Q(Y=y|X)=p\} = \{X \mid Q(Y=y|X)=p\}.
$$

We can apply the calibration assumption to replace the conditional term with $p$ and rewrite $\mathbb{P}(Q(Y = y \mid X) = p)$ as a sum over elements of $\mathbb{X}$. This gives:

\begin{align*}
    \mathlarger{\mathop{\mathbb{E}}}_{y \sim \mathbb{P}(Y)} \left[ g(y) \right] 
        &= \sum_{y \in \mathbb{Y}} g(y) \int_{0}^1 p \cdot \ \mathbb{P}(Q(Y = y \mid X) = p) \ dp \\
        &= \sum_{y \in \mathbb{Y}} g(y) \int_{0}^1 p \cdot \ \sum_{x \in \mathbb{X}} \mathbb{I}[Q(Y = y \mid X = x) = p] \cdot \mathbb{P}(X = x) \ dp \\
        &= \sum_{y \in \mathbb{Y}} g(y) \sum_{x \in \mathbb{X}} Q(Y = y \mid X = x)\cdot \mathbb{P}(X = x) \\
        &= \sum_{x \in \mathbb{X}} \mathbb{P}(X = x) \sum_{y \in \mathbb{Y}} g(y) \cdot  Q(Y = y \mid X = x) \\
        &= \mathlarger{\mathop{\mathbb{E}}}_{x \sim \mathbb{P}(X)} \ \ \mathlarger{\mathop{\mathbb{E}}}_{y \sim Q(Y | X = x)} \left[ g(y) \right].
\end{align*}
\end{proof}

\section{Additional Figures}
\begin{figure}[h]
    \centering
    \includegraphics[width=0.45\linewidth]{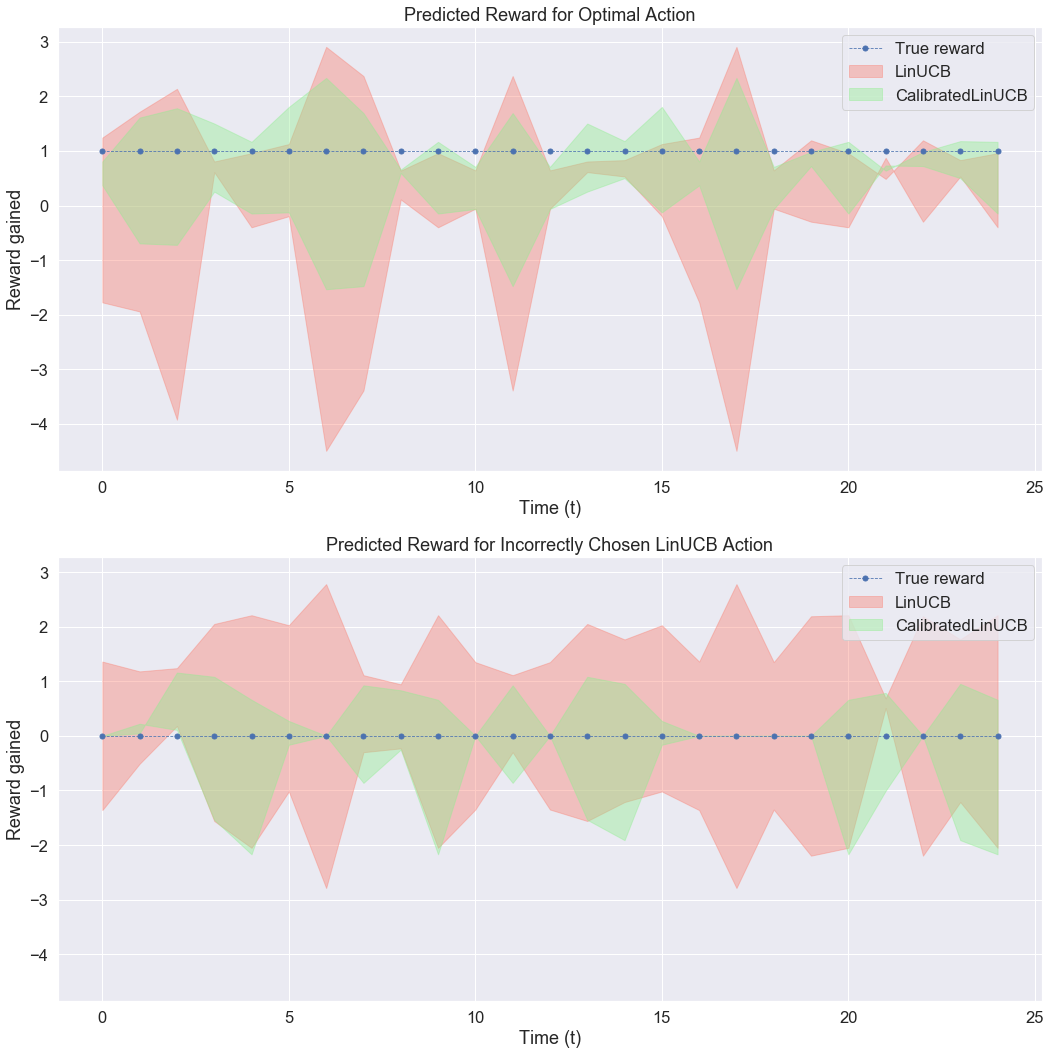} 
    \caption{Predicted expected reward for both LinUCB and CalLinUCB algorithms on the covertype dataset. Figures show predictions at random timesteps where CalLinUCB chose the optimal action but LinUCB did not. Top: Predicted reward of both algorithms for the optimal action. Bottom: Predicted reward of both algorithms for the action which the algorithm chose to pick instead of the optimal action at that timestep. We can see LinUCB consistently underestimates reward from optimal action and overestimates reward from other actions. On the other hand, CalLinUCB is more accurate in its uncertainty predictions.
    }
    \label{fig:covertype_predictions}
\end{figure}  

\begin{figure}[h!]
    \centering
    \includegraphics[width=0.45\linewidth]{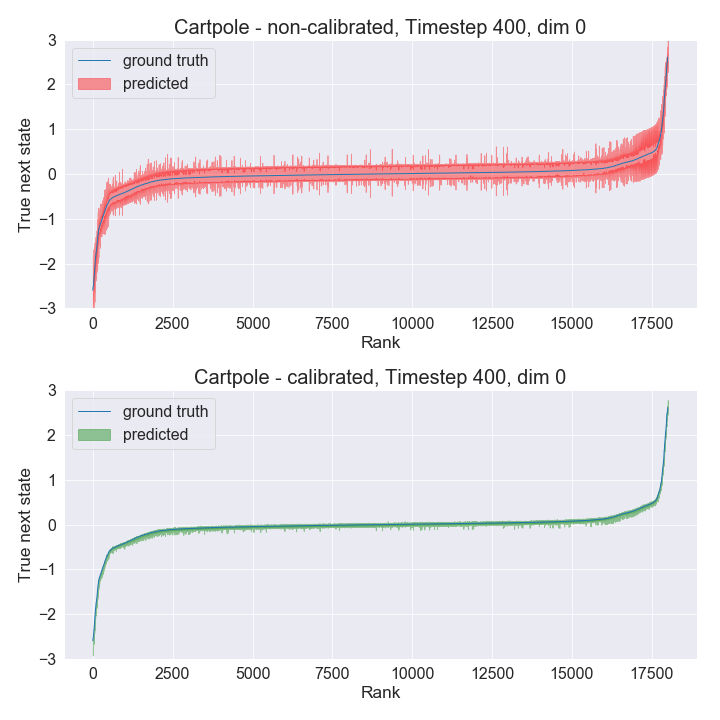}
    \includegraphics[width=0.45\linewidth]{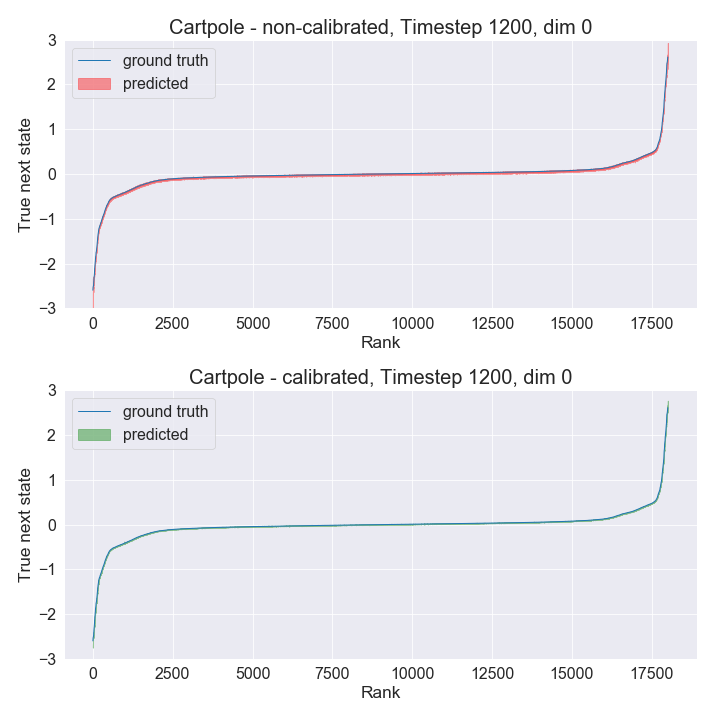}
    \caption{Cartpole future state predictions. The calibrated algorithm has much tighter uncertainties around the true next state in early training iterations. Later into training, their uncertainties are almost equivalent.}
\label{fig:cartpole_predictions}
\end{figure}